\pdfoutput=1
\documentclass[10pt]{article}

\usepackage[margin=1in]{geometry}
\usepackage{amsmath,amssymb,amsthm}
\usepackage{cite}
\usepackage{url}
\usepackage{graphicx}
\usepackage{marvosym}
\usepackage{tikz}
\usepackage{pgfplots}
\pgfplotsset{compat=1.18}
\usepackage{algorithm}
\usepackage{algorithmic}

\theoremstyle{plain}
\newtheorem{theorem}{Theorem}
\newtheorem{lemma}{Lemma}
\newtheorem{proposition}{Proposition}
\theoremstyle{definition}
\newtheorem{definition}{Definition}

\newcommand{\keywords}[1]{\par\medskip\noindent\textbf{Keywords:} #1}

\title{Directed Neuro-Symbolic Stochastic Execution for Verification of
Distributed Parallel AI Programs}

\author{Gautham Koorma \quad Vikas Sharma \quad George Edwards \quad
Mahdi Eslamimehr\textsuperscript{(\Letter)}\thanks{\textsuperscript{(\Letter)}
Corresponding author: \texttt{mahdi@quandarypeak.com}}\\[6pt]
Quandary Peak Research, Los Angeles, CA, USA\\[2pt]
\normalsize\texttt{gautham@quandarypeak.com, vikas@quandarypeak.com,}\\
\normalsize\texttt{george@quandarypeak.com, mahdi@quandarypeak.com}}
\date{}

\begin{document}
\maketitle

\begin{abstract}
Distributed parallel Artificial Intelligence (AI) programs expose reliability
gaps that conventional testing cannot close: parallel executions are
non-deterministic, and AI workloads bring high-dimensional inputs and
non-linear operations that defeat fuzzing and symbolic execution in
isolation.
We present Directed Neuro-Symbolic Stochastic Execution (DNSSE), a
hybrid testing framework that couples schedule prediction guided by a
Large Language Model (LLM) with symbolic constraint solving and
coverage-guided stochastic mutation. We model distributed AI executions as non-deterministic transition
systems, specify correctness in linear temporal logic, and prove soundness,
bounded completeness, and probabilistic completeness of the hybrid solver,
together with an expected-cost analysis of LLM-guided schedule exploration.
A scalable implementation on PyTorch and Ray detects $2.9\times$ more
concurrency bugs than the strongest baseline and raises average branch
coverage from 68.6\,\% to 91.6\,\% across five realistic distributed AI
benchmarks.
\end{abstract}

\section{Introduction}
\label{sec:intro}

Production AI systems such as data-parallel training pipelines,
federated learning services, reinforcement learners, and high-throughput
inference servers execute across heterogeneous clusters, coordinate
thousands of concurrent activities, and manipulate models with billions of
parameters. As these systems enter safety- and business-critical roles,
ensuring their reliability and correctness has become a formidable and
largely unsolved engineering challenge.

The difficulty stems from two compounding sources of complexity. First,
parallel programs are inherently non-deterministic: data races, deadlocks,
and atomicity violations manifest only under particular interleavings, and
empirical studies show that such faults are among the most frequent and
damaging defects in multithreaded applications~\cite{lu2008} and datacenter
distributed systems~\cite{leesatapornwongsa2016}. Second, AI programs have
high-dimensional continuous input spaces, and their control flow is guarded
by non-linear operations (activations, normalizations, tensor
contractions) that resist classical analysis. Where the two paradigms
intersect, the product of the interleaving space and the input space renders
the reachable state space intractable for conventional testing.

Existing techniques address at most one side of this product space. Symbolic
execution~\cite{cadar2008} offers rigorous per-path reasoning but suffers
from path explosion and from solver incompleteness on floating-point,
non-linear arithmetic. Coverage-guided and parallel
fuzzing~\cite{song2019,wang2021,zhou2025} scales well but rarely penetrates
regions guarded by tight numerical conditions, and it is oblivious to the
scheduling dimension. Systematic concurrency
testing~\cite{racesched2014} controls the scheduler but does
not reason about numerical inputs. Testing techniques for deep-learning
systems and libraries~\cite{pei2017,deng2022} target sequential APIs or the
model function itself, leaving the distributed program that embeds the model
out of scope.

To close this gap we propose \emph{Directed Neuro-Symbolic Stochastic
Execution} (DNSSE), which fuses three engines into one directed search: a
symbolic engine that solves the decidable fragment of path constraints, a
stochastic mutation engine that resolves the non-linear fragment by
coverage-guided randomized search, and a Large Language Model (LLM)
scheduler that reads the semantic structure of distributed AI code and
ranks thread interleavings by their estimated potential to expose
concurrency violations. The LLM acts only as a learned prior over the
schedule space; feasibility remains with the symbolic and stochastic
engines, so the formal guarantees never depend on the LLM being right.

This paper makes four contributions: (1)~DNSSE, a novel algorithm
combining LLM-guided schedule prediction with directed symbolic and
stochastic execution for distributed parallel AI programs
(Sects.~\ref{sec:framework}--\ref{sec:algorithm}); (2)~a formal framework
based on transition systems and linear temporal logic (LTL), with proofs
of soundness, bounded completeness, and probabilistic completeness, plus
an expected-cost bound quantifying when learned ranking beats uniform
exploration (Sect.~\ref{sec:guarantees}); (3)~a scalable architecture on
PyTorch and Ray that distributes tracing, constraint solving, and fuzzing
across a cluster (Sect.~\ref{sec:architecture}); and (4)~an evaluation on
five realistic distributed AI benchmarks, in which DNSSE detects
$2.9\times$ more concurrency bugs than the strongest baseline (73 vs.\ 25
in aggregate) and raises average branch coverage from 68.6\,\% to
91.6\,\% (Sect.~\ref{sec:evaluation}).

\section{Related Work}
\label{sec:related}

\paragraph{Testing Concurrent Programs.}
Foundations for reasoning about concurrent executions were laid by temporal
logic~\cite{pnueli1977,manna1981}, still the standard language for
specifying safety and liveness of interleaved computations. Precise dynamic race detection is
exemplified by FastTrack~\cite{flanagan2009}, while race-directed
scheduling combines dynamic analysis with schedule synthesis to force
concurrent programs into real races~\cite{racesched2014}; the directed
paradigm has been extended to deadlocks~\cite{sherlock2014}, atomicity
violations~\cite{atomicity2018}, and timing analysis of event-driven
programs~\cite{wcetdirected2015}. These techniques are effective for traditional Java/C/C++
code, but none of them reasons about the non-linear numerical guards and
tensor-level data flow that dominate AI workloads.

\paragraph{Symbolic Execution and Hybrid Testing.}
Symbolic executors such as KLEE~\cite{cadar2008} and concolic
testing~\cite{godefroid2005} enumerate feasible paths systematically by
solving accumulated path constraints. Because solving stalls on deep or non-linear
conditions, hybrid systems hand off work between a fuzzer and a symbolic
engine, as in QSYM~\cite{yun2018}. Orthogonally, parallel fuzzing scales
stochastic testing to clusters: P-Fuzz distributes fuzzing state through a
database-centric architecture~\cite{song2019}, AFL-EDGE partitions seeds
into mutually exclusive tasks~\cite{wang2021}, and KRAKEN adapts task
allocation to program characteristics~\cite{zhou2025}. All of these systems
target sequential binaries or thread-oblivious executions; the schedule
dimension of the search space is left unexplored.

\paragraph{Verification and Testing of AI Systems.}
A complementary line of work tests the AI artifact itself.
DeepXplore~\cite{pei2017} introduced coverage-guided whitebox testing of
neural networks, and DeepREL fuzzes deep-learning libraries through
automatically inferred relational APIs~\cite{deng2022}. On the
formal side, Reluplex extends SMT solving to verify piecewise-linear
networks~\cite{katz2017}. These techniques treat the model or the library
API as a sequential function; they do not address races, deadlocks, or
atomicity violations arising when models are trained or served by
concurrent distributed programs.

\paragraph{LLMs for Software Testing.}
LLMs have recently been embedded in the testing loop. TitanFuzz uses LLMs
as zero-shot generators and mutators of deep-learning API
programs~\cite{deng2023}, and Fuzz4All generalizes LLM-driven fuzzing
across languages and systems~\cite{xia2024}. Closer to program analysis,
LLM-powered symbolic execution decomposes path-constraint reasoning into
subtasks delegated to a model~\cite{li2025}, hybrid concolic testing
employs LLMs for guided path exploration~\cite{llmconcolic2026}, and
directed execution combined with LLM-driven analysis has been applied to
zero-day malware detection~\cite{llmmalware2026}. None of these efforts,
however, uses an LLM to navigate the \emph{schedule} space of a
distributed parallel AI program, nor do they integrate schedule prediction
with a hybrid symbolic--stochastic constraint solver. DNSSE fills
precisely this gap.

\section{Motivating Example}
\label{sec:example}

The listing below distills a data-parallel training loop in the style of
PyTorch \texttt{DistributedDataParallel}~\cite{li2020}: each worker
backpropagates and, only when the loss falls below a threshold,
updates a shared convergence counter that a coordinator
concurrently reads to decide termination.

\begin{center}
\begin{minipage}{0.9\textwidth}
\footnotesize
\begin{verbatim}
def worker(rank, model, loader, shared):
    for x, y in loader:
        out  = model(x)               # softmax, tanh: non-linear
        loss = cross_entropy(out, y)
        loss.backward()               # async all-reduce hooks
        if loss.item() < THETA:       # data-dependent guard
            t = shared.converged      # read shared counter
            shared.converged = t + 1  # unsynchronized write
        optimizer.step()
\end{verbatim}
\end{minipage}
\end{center}

The read--modify--write on \texttt{shared.converged} is not atomic: if
two workers interleave between the read and the write, an increment is
lost and the coordinator may never observe convergence. Exposing the
defect requires satisfying two coupled conditions at once. First, an
input $x$ must drive the loss below $\Theta$, a constraint routed
through the cross-entropy of a softmax that lies outside decidable SMT
theories, so purely symbolic tools stall. Second, two workers must
interleave within a window of a few instructions, which random
scheduling hits with vanishing probability, so even a fuzzer that finds
a low-loss input almost never observes the race. DNSSE divides the labor: the LLM scheduler flags the unsynchronized access pair as high
risk and ranks interleavings that place two workers inside the window
first; the symbolic engine tracks the linear index and control-flow
constraints exactly; and the stochastic engine searches for $x$ with
$\mathrm{loss}(x) < \Theta$ under a branch-distance objective.

\section{Formal Framework}
\label{sec:framework}

\subsection{System Model}

\begin{definition}[Distributed Parallel AI Program]
\label{def:program}
Let $\mathcal{P} = \langle \mathcal{T}, \mathcal{S}, \mathcal{I},
\mathcal{O}, \Delta, s_0 \rangle$ be a distributed parallel AI program,
where $\mathcal{T} = \{t_1, \dots, t_n\}$ is a
finite set of threads or distributed workers; $\mathcal{S}$ is the set of
global states, each comprising all local thread states, the shared memory
(e.g., model parameters and optimizer state), and channel contents;
$\mathcal{I} \subseteq \mathbb{R}^d$ is the input space; $\mathcal{O}$ is
the output space;
$\Delta \colon \mathcal{S} \times \mathcal{I} \times \mathcal{T}
\rightarrow 2^{\mathcal{S}}$ is the non-deterministic transition
relation; and $s_0 \in \mathcal{S}$ is the initial state.
\end{definition}

A thread $t$ is \emph{enabled} in $s$ under $x$ iff
$\Delta(s,x,t) \neq \emptyset$, with $\mathrm{En}(s,x)$ the set of
enabled threads.

\begin{definition}[Schedule and Trace]
\label{def:trace}
A \emph{schedule} of length $k$ is a sequence
$\sigma = t_{i_0} t_{i_1} \cdots t_{i_{k-1}} \in \mathcal{T}^{k}$.
Given an input $x \in \mathcal{I}$, the pair $(\sigma, x)$ induces the set
of \emph{traces}
$\pi = s_0 \xrightarrow{t_{i_0}} s_1 \xrightarrow{t_{i_1}} \cdots
\xrightarrow{t_{i_{k-1}}} s_k$
with $s_{j+1} \in \Delta(s_j, x, t_{i_j})$ and $t_{i_j} \in
\mathrm{En}(s_j,x)$ for all $j$. We write $\Pi_D(\mathcal{P})$ for the set
of traces of length at most $D$.
\end{definition}

Each transition is labeled with the memory \emph{events} it performs:
$\mathrm{Rd}(t,v)$ and $\mathrm{Wr}(t,v)$ denote thread $t$ reading or
writing shared location $v$ (tensors carry their accessed index sets). The
\emph{happens-before} relation $\prec_\pi$ of a trace $\pi$ is the
smallest partial order over events containing program order and
synchronization order (lock release to subsequent acquire, send to
receive, collective barriers)~\cite{flanagan2009}.

\begin{definition}[Concurrency Faults]
\label{def:faults}
Let $\pi$ be a trace of $\mathcal{P}$. (i)~$\pi$ exhibits a \emph{data
race} on $v$ iff it contains $e_1 = \mathrm{Wr}(t_i, v)$ and
$e_2 \in \{\mathrm{Rd}(t_j, v), \mathrm{Wr}(t_j, v)\}$, $i \neq j$,
unordered by $\prec_\pi$. (ii)~$\pi$ exhibits an \emph{atomicity
violation} iff a block declared atomic in $t_i$ is interleaved by a
conflicting access of $t_j$, $j \neq i$, in a non-serializable order.
(iii)~A state $s$ is a \emph{deadlock} iff $\mathrm{En}(s,x) = \emptyset$
while some thread has not terminated.
\end{definition}

\subsection{Symbolic Executions and Path Constraints}

During symbolic execution the input is a vector of symbolic variables
$X = [X_1, \dots, X_d]^{\mathsf T}$; every branch along a trace $\pi$
contributes its condition (or negation), forming the \emph{path
constraint}
\begin{equation}
    \Phi_\pi(X) \;=\; \bigwedge_{k=0}^{|\pi|-1} c_k(X),
    \label{eq:pathconstraint}
\end{equation}
where $c_k(X)$ is the branch condition at step $k$ over the symbolic
store. A trace is \emph{feasible} iff some $x \in \mathcal{I}$ satisfies
$\Phi_\pi(x)$; any such \emph{witness} $x$, paired with the schedule
$\sigma$, is a replayable test case.

\subsection{Correctness Properties in LTL}

Let $AP$ be a set of atomic propositions over $\mathcal{S}$. LTL formulas
are built from $AP$ with Boolean connectives and the temporal operators
$\square$ (always), $\lozenge$ (eventually), and $\bigcirc$ (next), with
standard semantics over traces~\cite{baier2008}. Freedom from data races and from deadlock are the invariants
\begin{eqnarray}
    \phi_{\mathrm{race}} &=& \square\, \neg
    \bigvee_{t_i \neq t_j,\; v}
    \bigl[ \mathrm{Wr}(t_i, v) \wedge \mathrm{Acc}(t_j, v) \wedge
    \mathrm{unord}(t_i, t_j, v) \bigr],
    \label{eq:racefree} \\
    \phi_{\mathrm{live}} &=& \square
    \bigl( \mathrm{run} \rightarrow
    {\textstyle\bigvee_{t \in \mathcal{T}}}\, \mathrm{En}(t) \bigr),
    \label{eq:deadlockfree}
\end{eqnarray}
where $\mathrm{Acc}$ is a read or write access and $\mathrm{unord}$
states that the two accesses are unordered by $\prec_\pi$.
Atomicity of a block $\beta$ requires that between
$\mathrm{begin}(\beta)$ and $\mathrm{end}(\beta)$ no conflicting access
of another thread occurs in a non-serializable order
(Definition~\ref{def:faults}). The verification goal is to decide whether
some $\pi \in \Pi_D(\mathcal{P})$ violates $\phi$ (written
$\pi \not\models \phi$) and, if so, to produce the witness pair
$(\sigma, x)$.

\subsection{The Hybrid Constraint-Solving Problem}
\label{sec:hybridproblem}

Path constraints of AI programs mix tractable and intractable conjuncts:
array indexing, batching logic, and synchronization guards yield linear
arithmetic, whereas activation functions such as $\tanh$ and
$\mathrm{softmax}$ yield non-linear transcendental terms undecidable for
standard SMT theories. Let $\mathcal{L}$ denote the decidable fragment
supported by the solver (quantifier-free linear real and integer
arithmetic with arrays). We define a syntactic classifier
$\tau(c) \in \{\mathrm{sym}, \mathrm{stoch}\}$ that assigns each conjunct
$c$ to $\mathrm{sym}$ iff every term of $c$ lies in $\mathcal{L}$, and
partition
\begin{equation}
    \Phi_\pi(X) \;=\; \Phi_{\mathrm{sym}}(X) \,\wedge\,
    \Phi_{\mathrm{stoch}}(X),
    \qquad
    \Phi_{\mathrm{sym}} = \!\!\bigwedge_{\tau(c_k)=\mathrm{sym}}\!\! c_k,
    \quad
    \Phi_{\mathrm{stoch}} = \!\!\bigwedge_{\tau(c_k)=\mathrm{stoch}}\!\! c_k.
    \label{eq:partition}
\end{equation}
The \emph{hybrid constraint-solving problem} is: given a property $\phi$,
find a schedule $\sigma^{*}$ and input $x^{*}$ such that
\begin{equation}
    \Phi_{\mathrm{sym}}(x^{*}) \,\wedge\, \Phi_{\mathrm{stoch}}(x^{*})
    \,\wedge\, \mathrm{feasible}(\sigma^{*})
    \;\equiv\; \mathrm{true}
    \quad\text{and}\quad
    \pi(\sigma^{*}\!, x^{*}) \not\models \phi .
    \label{eq:hybridproblem}
\end{equation}
DNSSE attacks Eq.~(\ref{eq:hybridproblem}) with an SMT solver for
$\Phi_{\mathrm{sym}}$, randomized search for $\Phi_{\mathrm{stoch}}$, and an
LLM prior over candidate schedules $\sigma$.

\section{The DNSSE Algorithm}
\label{sec:algorithm}

DNSSE orchestrates three components (a symbolic engine, a stochastic
mutation engine, and an LLM-driven path scheduler) in one distributed
worklist search over the state space of Definition~\ref{def:program}.

\subsection{Core Algorithm}

Algorithm~\ref{alg:dnsse} presents the core loop. The worklist
$\mathcal{W}$ holds tuples $(s, \Phi, \sigma, x)$: a frontier state, its
accumulated path constraint, the schedule prefix that reached it, and a
concrete witness input $x$ satisfying $\Phi$, an invariant central to
the soundness proof (Sect.~\ref{sec:guarantees}). Each iteration selects the
highest-priority entry, checks the property, and expands the frontier: the
LLM ranks the enabled transitions, each ranked transition is executed
symbolically, and the extended constraint is solved by SMT when it lies in
the decidable fragment and by stochastic search otherwise. Constraints not
satisfied within the mutation budget are re-enqueued with decayed priority
rather than discarded, so prioritization never erases part of the search
space.

\begin{algorithm}[t]
\caption{Directed Neuro-Symbolic Stochastic Execution}
\label{alg:dnsse}
\footnotesize
\begin{algorithmic}[1]
\REQUIRE Program $\mathcal{P}$, LTL property $\phi$, LLM $\mathcal{M}$,
depth bound $D$
\ENSURE Counterexample trace $\pi_{\mathrm{bug}}$ with witness
$(\sigma, x)$, or $\emptyset$
\STATE $\mathcal{W} \leftarrow \{ (s_0, \mathrm{true}, \varepsilon,
x_{\mathrm{seed}}) \}$
\WHILE{$\mathcal{W} \neq \emptyset$}
    \STATE $(s, \Phi, \sigma, x) \leftarrow
    \mathrm{SelectState}(\mathcal{W}, \mathcal{M})$
    \IF{$\pi(s) \not\models \phi$}
        \RETURN $\pi(s)$ \COMMENT{bug found; $(\sigma, x)$ replays it}
    \ENDIF
    \IF{$|\sigma| \geq D$} \STATE \textbf{continue} \ENDIF
    \STATE $\mathcal{E} \leftarrow \mathrm{GetEnabledTransitions}(s)$
    \STATE $\sigma_{\mathrm{next}} \leftarrow
    \mathcal{M}.\mathrm{PredictSchedule}(\mathcal{E}, s, \phi)$
    \FOR{each transition $e \in \sigma_{\mathrm{next}}$}
        \STATE $(s', c) \leftarrow \mathrm{ExecSymbolic}(s, e)$;\ \ $\Phi' \leftarrow \Phi \wedge c$
        \IF{$\tau$-classification of $\Phi'$ is fully decidable}
            \STATE $x' \leftarrow \mathrm{SMTSolve}(\Phi')$
        \ELSE
            \STATE $x' \leftarrow
            \mathrm{StochasticFuzz}(\Phi', s', x)$
        \ENDIF
        \IF{$x' \neq \mathrm{UNSAT}$}
            \STATE $\mathcal{W} \leftarrow \mathcal{W} \cup
            \{ (s', \Phi', \sigma \cdot e, x') \}$
        \ENDIF
    \ENDFOR
\ENDWHILE
\RETURN $\emptyset$
\end{algorithmic}
\end{algorithm}

\subsection{LLM-Guided Schedule Prediction}
\label{sec:llmsched}

The LLM scheduler differentiates DNSSE from classical directed testing.
Given state $s$, enabled transitions $\mathcal{E}$, and property $\phi$,
the model receives a serialized context (thread program counters, lock
and channel state, pending collectives, and the source fragments
adjacent to each enabled transition) and generates a ranking of
$\mathcal{E}$.
If the returned ranking is malformed or incomplete, DNSSE falls back to a
uniform random order, so scheduling always yields a \emph{permutation} of
$\mathcal{E}$; Sect.~\ref{sec:guarantees} exploits this. The final
priority combines the normalized LLM confidence $P_{\mathcal{M}}(e)$ with
a structural risk heuristic:
\begin{equation}
    \mathrm{score}(e) \;=\; \alpha \, P_{\mathcal{M}}(e) +
    (1-\alpha) \, H(e,s),
    \qquad
    H(e, s) \;=\; \frac{\mathrm{sh}(e)}{\mathrm{ops}(e)} \cdot
    \frac{|\mathcal{T}_{\mathrm{acc}}(e,s)|}{|\mathcal{T}|},
    \label{eq:heuristic}
\end{equation}
where $\mathrm{sh}(e)$ and $\mathrm{ops}(e)$ count shared-variable
accesses and total operations of $e$,
$\mathcal{T}_{\mathrm{acc}}(e,s)$ is the set of threads concurrently
accessing the memory region $e$ touches, and $\alpha \in [0,1]$ balances
learned against structural evidence.

\subsection{Stochastic Constraint Solving}
\label{sec:stochsolve}

For constraints containing non-linear conjuncts, DNSSE minimizes the
classical branch-distance objective~\cite{korel1990}. Each conjunct
$c$ of $\Phi_{\mathrm{stoch}}$ contributes a non-negative distance
$d(c, x)$ that is zero iff $c$ holds under $x$ (e.g.,
$d(a \leq b) = \max(0,\, a - b)$, evaluated on the concrete execution),
and the solver seeks
\begin{equation}
    x^{*} \;=\; \arg\min_{x \in \mathcal{I}}\;
    F(x), \qquad
    F(x) \;=\; \sum_{c \,\in\, \Phi_{\mathrm{stoch}}} d(c, x)
    \;\;\text{s.t.}\;\; \Phi_{\mathrm{sym}}(x),
    \label{eq:fitness}
\end{equation}
accepting $x^{*}$ as a witness iff $F(x^{*}) = 0$. The search perturbs
the current witness with coverage-guided mutations (Gaussian noise,
boundary values, dimension-wise crossover of corpus seeds), projecting
candidates onto the subspace satisfying $\Phi_{\mathrm{sym}}$ when the
projection is linear and re-validating them with the SMT solver
otherwise. Every accepted witness comes from a concrete replayed
execution, so floating-point semantics are exact by construction and no
solver-level approximation of transcendental functions is needed.

\paragraph{Complexity.}
Let $B$ bound the per-thread branching factor, $n = |\mathcal{T}|$, and
$D$ the depth bound. Each step chooses one of at most $n$ enabled threads
and one of at most $B$ branch outcomes, so
$|\Pi_D(\mathcal{P})| = O\!\bigl((Bn)^{D}\bigr)$: the schedule and input
dimensions multiply, which is precisely the explosion that defeats
single-paradigm tools. DNSSE does not shrink this worst case, as no
sound and complete method can, but it changes the \emph{order} of the
visit;
Proposition~\ref{prop:cost} quantifies the expected saving under a
budget. Per node, one SMT query over conjunctive
$\mathrm{QF\_LRA}$ constraints is decidable in polynomial time (NP-hard
once integers enter), one stochastic solve costs $O(k \cdot d)$ for $k$
mutation rounds on $d$-dimensional inputs, and the LLM adds one
bounded-length inference whose empirical share of runtime is measured in
Sect.~\ref{sec:evaluation}.

\section{Formal Guarantees}
\label{sec:guarantees}

Throughout, $\mathcal{P}$, $\phi$, and $D$ are fixed, and ``DNSSE
returns $\pi_{\mathrm{bug}}$'' means Algorithm~\ref{alg:dnsse}
terminates at its property-violation return.

\begin{lemma}[Partition Correctness]
\label{lem:partition}
For every path constraint $\Phi_\pi(X)$, the partition of
Eq.~(\ref{eq:partition}) is a logical equivalence:
$\Phi_\pi \equiv \Phi_{\mathrm{sym}} \wedge \Phi_{\mathrm{stoch}}$, and
$\Phi_\pi$ is satisfiable iff
$\Phi_{\mathrm{sym}} \wedge \Phi_{\mathrm{stoch}}$ is satisfiable.
\end{lemma}
\begin{proof}
Each conjunct of Eq.~(\ref{eq:pathconstraint}) is mapped by the
classifier $\tau$ to exactly one sub-conjunction, with no other change,
so both sides of Eq.~(\ref{eq:partition}) contain the same conjuncts;
by associativity and commutativity of $\wedge$ they are logically
equivalent and thus equisatisfiable.
\end{proof}

\begin{lemma}[Witness Invariant]
\label{lem:witness}
Every tuple $(s, \Phi, \sigma, x)$ ever inserted into $\mathcal{W}$
satisfies $\Phi(x) = \mathrm{true}$, and $x$ replayed under schedule
$\sigma$ drives $\mathcal{P}$ from $s_0$ to $s$.
\end{lemma}
\begin{proof}
By induction on insertions. The initial tuple carries
$\Phi = \mathrm{true}$, satisfied by any seed. A tuple
$(s', \Phi', \sigma \cdot e, x')$ is inserted only after the guard
$x' \neq \mathrm{UNSAT}$; both solver branches return $x'$ only when
$\Phi'(x')$ holds: the SMT solver by producing a model, and the
stochastic solver by concretely executing the program and checking
$F(x') = 0$ in Eq.~(\ref{eq:fitness}), with
Lemma~\ref{lem:partition} combining the fragments. Symbolic execution of $e$ from $s$ extends the trace by
exactly $e$, so replaying $\sigma \cdot e$ under $x'$ reaches $s'$.
\end{proof}

\begin{theorem}[Soundness]
\label{thm:soundness}
If DNSSE returns $\pi_{\mathrm{bug}}$, then $\pi_{\mathrm{bug}}$ is a
feasible execution of $\mathcal{P}$ that violates $\phi$, and the returned
pair $(\sigma, x)$ deterministically replays it.
\end{theorem}
\begin{proof}
The returned trace is $\pi(s)$ for the selected tuple. By
Lemma~\ref{lem:witness}, its witness $x$ satisfies the accumulated
constraint and replaying $\sigma$ under $x$ reproduces exactly the
transitions of $\pi(s)$, so $\pi_{\mathrm{bug}} \in \Pi_D(\mathcal{P})$
is feasible, not merely symbolically consistent; and the property check
verified $\pi(s) \not\models \phi$ before returning. Hence the output is
a genuine counterexample and DNSSE reports no false positives.
\end{proof}

\begin{theorem}[Bounded Completeness]
\label{thm:completeness}
Assume (i) the SMT solver is sound and complete for $\mathcal{L}$ and
(ii) the stochastic solver is an oracle for $\Phi_{\mathrm{stoch}}$,
returning a witness whenever one exists. Then DNSSE explores every
feasible trace of length at most $D$; in particular, if some
$\pi \in \Pi_D(\mathcal{P})$ violates $\phi$, DNSSE returns a
counterexample.
\end{theorem}
\begin{proof}
By induction on trace length $\ell \leq D$, every feasible prefix is
eventually dequeued. The base case is the initial tuple. When a prefix is
expanded, the expansion loop iterates over the \emph{entire} ranked list,
which by construction (Sect.~\ref{sec:llmsched}, uniform fallback) is a
permutation of all enabled transitions; prioritization reorders but
never removes.
Satisfiability of each extension is decided exactly: decidable constraints
by assumption~(i), mixed constraints by assumption~(ii) with
Lemma~\ref{lem:partition}. Hence exactly the feasible extensions of
length $\ell+1$ are inserted, and none is starved, because re-enqueued
entries retain positive priority and the depth bound makes the tree
finite ($O((Bn)^D)$ nodes). A violating trace within depth $D$ is
therefore eventually dequeued and detected by the property check.
\end{proof}

Assumption (ii) idealizes the fuzzer; the next theorem replaces it with a
quantitative statement about the real, randomized solver.

\begin{theorem}[Probabilistic Completeness of Stochastic Solving]
\label{thm:probabilistic}
Let $\Phi = \Phi_{\mathrm{sym}} \wedge \Phi_{\mathrm{stoch}}$ be
satisfiable and let
$S = \{ x \in \mathcal{I} \mid \Phi(x) \}$ denote its witness set. Suppose
each mutation round draws its candidate from a proposal distribution $q$
with $q(S) \geq \varepsilon > 0$ (full support, e.g., a Gaussian kernel
mixed with uniform restarts). Then the probability that $k$ independent
rounds all fail to find a witness is at most $(1-\varepsilon)^{k}$, so the
solver succeeds with probability at least $1 - (1-\varepsilon)^{k}
\rightarrow 1$ as $k \rightarrow \infty$, and the expected number of
rounds to success is at most $1/\varepsilon$.
\end{theorem}
\begin{proof}
Each round independently lands in $S$ with probability at least
$\varepsilon$, so $\Pr[\text{no witness in } k] \leq (1-\varepsilon)^k$;
the success probability is its complement, and a geometric random
variable with parameter $\varepsilon$ has expectation $1/\varepsilon$.
Coverage guidance and $\Phi_{\mathrm{sym}}$-projection only reallocate
proposal mass, while uniform restarts preserve $q(S) \geq \varepsilon$,
so the bound also covers the guided search.
\end{proof}

Finally, we quantify what the LLM buys. Prioritization cannot enlarge the
set of explorable traces (Theorem~\ref{thm:completeness}); its value is
the \emph{expected cost} to the first counterexample under a finite
budget.

\begin{proposition}[Expected Exploration Cost under Ranking]
\label{prop:cost}
Let $\Sigma$ be the set of maximal schedules explored at some frontier,
$N = |\Sigma|$, of which a nonempty subset $\mathcal{B} \subseteq \Sigma$
exposes the target violation. Uniform random exploration without
replacement dequeues its first element of $\mathcal{B}$ after
$(N+1)/(|\mathcal{B}|+1)$ schedules in expectation. If with probability
$p$ the ranked order places some element of $\mathcal{B}$ within the first
$m$ positions, the expected number of schedules explored before exposure
is at most
\begin{equation}
    \mathbb{E}[\mathrm{cost}] \;\leq\; p\, m + (1-p)\, N .
    \label{eq:cost}
\end{equation}
For a single buggy schedule ($|\mathcal{B}| = 1$) the uniform baseline
is $(N+1)/2$, so any ranker with $p > 1/2$ and $m < N/2$ strictly wins,
and the advantage grows linearly in $N$ when $m \ll N$.
\end{proposition}
\begin{proof}
The uniform figure is the standard expectation of the minimum position of
$|\mathcal{B}|$ marked items in a uniformly random permutation of $N$
items. For the bound, condition on the ranking event: with probability
$p$ exposure occurs within the first $m$ dequeues; otherwise it occurs
after at most $N$. Taking expectations yields
Eq.~(\ref{eq:cost}).
\end{proof}

\section{System Architecture and Implementation}
\label{sec:architecture}

DNSSE is implemented in Python~3.11 on PyTorch~2.1~\cite{paszke2019} and
the Ray~2.9 distributed computing framework~\cite{moritz2018}.
Figure~\ref{fig:architecture} shows the architecture, a master--worker
design coordinated through Ray's actor model.

\begin{figure}[t]
\centering
\begin{tikzpicture}[scale=0.68, transform shape, node distance=2cm, auto,
    block/.style={draw, rectangle, minimum width=2.9cm, minimum
    height=0.85cm, rounded corners=2pt}]
    \node [block, fill=blue!12] (llm) {LLM Scheduler ($\mathcal{M}$)};
    \node [block, fill=green!12, below of=llm, node distance=1.7cm]
    (coord) {Central Coordinator};
    \node [block, fill=red!10, below left of=coord, node distance=3.0cm]
    (worker1) {Ray Worker 1};
    \node [block, fill=red!10, below right of=coord, node distance=3.0cm]
    (worker2) {Ray Worker $N$};
    \node [block, fill=yellow!15, below of=worker1, node distance=1.35cm]
    (smt1) {Z3 SMT Solver};
    \node [block, fill=orange!15, below of=worker2, node distance=1.35cm]
    (fuzz2) {LibFuzzer Engine};
    \node [block, fill=violet!10, below of=coord, node distance=4.7cm]
    (redis) {Redis State Store};
    \draw[<->, thick] (llm) -- node[right, font=\scriptsize] {schedule}
    (coord);
    \draw[->, thick] (coord) -| node[above left, font=\scriptsize] {task}
    (worker1);
    \draw[->, thick] (coord) -| node[above right, font=\scriptsize] {task}
    (worker2);
    \draw[<->, thick] (worker1) -- (smt1);
    \draw[<->, thick] (worker2) -- (fuzz2);
    \draw[dashed] (worker1) -- (worker2) node[midway, above,
    font=\scriptsize] {shared state};
    \draw[<->, thick] (redis) -| (smt1);
    \draw[<->, thick] (redis) -| (fuzz2);
\end{tikzpicture}
\caption{System architecture of the DNSSE framework.}
\label{fig:architecture}
\end{figure}
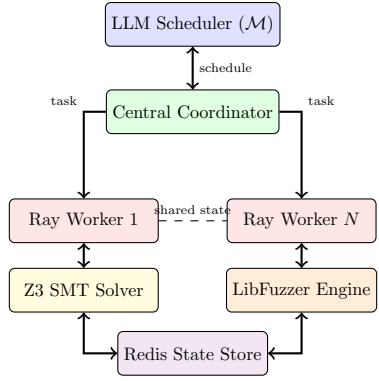

\paragraph{LLM Scheduler.}
The scheduler is a Llama~3 8B model~\cite{grattafiori2024} served
through vLLM~\cite{kwon2023} with 4-bit quantization, fine-tuned on
12{,}000 annotated concurrency-bug traces from open-source distributed
AI projects, each pairing a serialized scheduling context with the
transition that led to a confirmed violation. The inference prompt
encodes the execution state, enabled transitions, and target LTL
property; the model emits a JSON-formatted ranking, with the uniform
fallback of Sect.~\ref{sec:llmsched} engaging on malformed output.

\paragraph{Coordinator and State Store.}
A singleton Ray actor maintains the global worklist $\mathcal{W}$ as a
priority queue keyed by the scores of Eq.~(\ref{eq:heuristic}) and
dispatches expansion tasks through Ray's asynchronous task API; workers
are stateless between tasks, so a crashed worker is simply restarted and
its task re-issued, preserving the invariant of Lemma~\ref{lem:witness}.
A centralized Redis instance holds the global coverage bitmap,
deduplicated bug reports, and a cache mapping canonicalized constraints
to witnesses; workers merge coverage deltas idempotently, suppressing
redundant exploration without global locking.

\paragraph{Execution Workers.}
Each worker runs a sandboxed instance of the target program, intercepting
the PyTorch computational graph with \texttt{torch.fx} to extract
symbolic expressions during the forward pass. Conditions over standard
arithmetic are classified by $\tau$ (Sect.~\ref{sec:hybridproblem}) as
decidable and forwarded to Z3~4.12~\cite{demoura2008}; conjuncts
involving activation functions and tensor contractions go to an embedded
LibFuzzer engine implementing the search of Eq.~(\ref{eq:fitness}).
Schedule control instruments synchronization points (locks, queues,
collectives) with a cooperative scheduler that realizes the prefix
$\sigma$ deterministically.

\paragraph{Reproducibility.}
Experiments run on an Ubuntu~22.04 cluster with 8 NVIDIA A100 (80\,GB)
GPUs and 64 CPU cores, with pinned dependencies (PyTorch~2.1, Ray~2.9,
Z3~4.12, vLLM~0.3, Transformers~4.38) in Docker containers, LLVM~IR
for C/C++ runtime components generated by Clang~17, and fixed random
seeds. We plan to release DNSSE under the MIT license on GitHub; at this time,
the source code is publicly
available.\footnote{DNSSE temporary codebase host:
\url{https://drive.google.com/drive/folders/1o43AblLAnRNejZGa2_cpGUJN3el3Kon-?usp=drive_link}}

\section{Experimental Evaluation}
\label{sec:evaluation}

Our evaluation answers three research questions. \textbf{RQ1:} How does
DNSSE compare with state-of-the-art methods in coverage and bug
detection? \textbf{RQ2:} What overhead do the LLM scheduler and hybrid
solver introduce? \textbf{RQ3:} How much does each component contribute?

\subsection{Experimental Setup}

\paragraph{Baselines.}
We compare against (1)~\textbf{KLEE-AI}, an adaptation of
KLEE~\cite{cadar2008} extended to tensor operations through LLVM bitcode
instrumentation (pure symbolic execution);
(2)~\textbf{P-Fuzz}~\cite{song2019}, a distributed grey-box fuzzer with 8
parallel instances (scalable stochastic testing); and
(3)~\textbf{Random}, uniformly random schedules and inputs.

\paragraph{Benchmarks.}
The suite comprises five applications representative of production
distributed AI workloads: \textbf{Dist-Train}, data-parallel training
with gradient synchronization across 4 workers in the style of
DistributedDataParallel~\cite{li2020} (8{,}200 LOC); \textbf{Fed-Learn},
a federated aggregation server with asynchronous, differentially private
client updates (6{,}400 LOC); \textbf{RL-Agent}, multi-agent
reinforcement learning with shared replay buffers (5{,}900 LOC);
\textbf{Graph-GNN}, a distributed graph neural network with partitioned
message passing (7{,}100 LOC); and \textbf{Infer-Serve}, a concurrent
inference server with dynamic batching and model sharding (9{,}300 LOC).
Each experiment runs with a 1-hour wall-clock budget (KLEE-AI
additionally receives 24 hours); reported numbers average 5 independent
runs with fixed seeds.

\subsection{Bug Detection and Coverage (RQ1)}

Table~\ref{tab:results} reports branch coverage and confirmed concurrency
bugs (data races, deadlocks, atomicity violations). DNSSE outperforms
every baseline on every benchmark: in aggregate it detects 73 bugs versus
25 for P-Fuzz ($2.9\times$), 12 for KLEE-AI, and 7 for Random, and it
raises mean branch coverage to 91.6\,\% versus 68.6\,\%, 55.5\,\%, and
43.4\,\%, respectively. Random and P-Fuzz consumed the full 1-hour budget
and KLEE-AI timed out at 24 hours on every benchmark, whereas DNSSE
completed exploration in 1{,}590--2{,}680\,s (mean 2{,}134\,s).

\begin{table}[t]
\caption{Branch coverage and confirmed concurrency bugs per method. T/O:
KLEE-AI exceeded its extended 24-hour budget.}
\label{tab:results}
\begin{center}
\begin{tabular}{lcccccccc}
\hline\noalign{\smallskip}
& \multicolumn{4}{c}{Coverage (\%)} & \multicolumn{4}{c}{Bugs} \\
Benchmark & Rand. & KLEE-AI & P-Fuzz & \textbf{DNSSE} & Rand. & KLEE-AI
& P-Fuzz & \textbf{DNSSE} \\
\noalign{\smallskip}\hline\noalign{\smallskip}
Dist-Train & 42.1 & 58.3 & 65.4 & \textbf{89.2} & 1 & 3 & 5 & \textbf{14} \\
Fed-Learn & 38.5 & 49.1 & 71.2 & \textbf{93.7} & 0 & 2 & 4 & \textbf{11} \\
RL-Agent & 45.0 & 52.4 & 68.9 & \textbf{91.5} & 2 & 1 & 6 & \textbf{17} \\
Graph-GNN & 40.2 & 55.7 & 62.1 & \textbf{88.4} & 1 & 2 & 3 & \textbf{12} \\
Infer-Serve & 51.3 & 61.8 & 75.4 & \textbf{95.1} & 3 & 4 & 7 & \textbf{19} \\
\noalign{\smallskip}\hline\noalign{\smallskip}
Mean / Total & 43.4 & 55.5 & 68.6 & \textbf{91.6} & 7 & 12 & 25 &
\textbf{73} \\
\noalign{\smallskip}\hline
\end{tabular}
\end{center}
\end{table}

Per benchmark, DNSSE adds 19.7--26.3 coverage percentage points over
P-Fuzz (26.1\,\%--42.4\,\% relative) and up to 55.2 points over random
scheduling (Fed-Learn). The largest relative gain occurs on Graph-GNN,
whose deep guards in the message-passing logic are rarely penetrated by
P-Fuzz; the LLM
contributes most visibly on Infer-Serve, steering exploration through
the batching and sharding logic. Mechanistically, KLEE-AI times out on
paths guarded by non-linear activations, while P-Fuzz lacks semantic
awareness of the scheduling space and misses interleaving-dependent
defects such as the lost-update race of Sect.~\ref{sec:example}. Early completion on every benchmark reflects
Proposition~\ref{prop:cost}: directed scheduling reaches the interesting
region of the schedule space early.

\subsection{Performance Overhead (RQ2)}

Table~\ref{tab:overhead} breaks down where DNSSE spends its time. The
LLM scheduler accounts for 12.0\,\%--16.4\,\% of execution time (14.3\,\%
on average), an overhead more than repaid by the pruning of unproductive
exploration reflected in the RQ1 completion times. The stochastic solver
dominates on benchmarks rich in non-linear operations (51.2\,\% on
Graph-GNN), whereas the SMT share peaks on the predominantly linear
Infer-Serve (39.7\,\%), evidence that the classifier $\tau$ routes
constraints as intended.

\begin{table}[t]
\caption{DNSSE component time breakdown (\% of total).}
\label{tab:overhead}
\begin{center}
\begin{tabular}{lcccc}
\hline\noalign{\smallskip}
Benchmark & LLM & SMT & Fuzz & Other \\
\noalign{\smallskip}\hline\noalign{\smallskip}
Dist-Train & 15.2 & 28.4 & 41.7 & 14.7 \\
Fed-Learn & 13.8 & 31.2 & 38.9 & 16.1 \\
RL-Agent & 14.1 & 22.6 & 48.3 & 15.0 \\
Graph-GNN & 16.4 & 18.9 & 51.2 & 13.5 \\
Infer-Serve & 12.0 & 39.7 & 33.1 & 15.2 \\
\noalign{\smallskip}\hline
\end{tabular}
\end{center}
\end{table}

\subsection{Ablation Study (RQ3)}

Disabling individual modules isolates their contributions. Replacing the
LLM scheduler with random scheduling reduces bug detection by 58\,\%,
confirming that semantic schedule guidance, not merely hybrid
constraint solving, drives the concurrency results. Disabling the stochastic solver
(SMT only) causes timeouts on 3 of 5 benchmarks, reproducing the KLEE-AI
failure mode. Removing the SMT solver (fuzzing only) lowers coverage by
31\,\% on average: random mutation solves tight linear conditions only
with vanishing probability, as the $1/\varepsilon$ expectation of
Theorem~\ref{thm:probabilistic} predicts for witness sets of small
measure. The components are thus complementary, each covering a failure
mode of the other two.

\subsection{Discussion, Limitations, and Threats to Validity}

DNSSE's guarantees are conditional in two respects. First,
Theorem~\ref{thm:completeness} idealizes the stochastic solver; on real
budgets DNSSE inherits only the probabilistic guarantee of
Theorem~\ref{thm:probabilistic}, so a stochastic \emph{failure} never
proves infeasibility, so DNSSE re-enqueues rather than prunes such
branches. Second, the LLM affects efficiency, not correctness: a bad
ranking degrades Eq.~(\ref{eq:cost}) toward the uniform baseline but
cannot induce false positives (Theorem~\ref{thm:soundness}). The main
practical costs are LLM inference latency and the one-time fine-tuning
corpus of curated bug traces.

Regarding validity: our benchmarks, though modeled on production
workloads (36{,}900 LOC total), are curated constructions, and results
may differ on industrial codebases (external); LLM inference and the
stochastic solver are randomized, mitigated by averaging 5 runs with
fixed seeds (internal); and bug counts are deduplicated by distinct
happens-before configuration (Definition~\ref{def:faults}), avoiding
inflation by re-manifestations of one root cause (construct).

\section{Conclusion}
\label{sec:conclusion}

We introduced DNSSE, a hybrid framework for verifying distributed
parallel AI programs that integrates LLM-guided schedule prediction with
symbolic execution and coverage-guided stochastic mutation, on an
explicit formal foundation: a transition-system model with LTL
specifications, proofs of soundness and bounded completeness, a
probabilistic-completeness bound for the stochastic solver, and an
expected-cost analysis of learned schedule ranking. Empirically, DNSSE
detects $2.9\times$ more concurrency bugs than the strongest baseline and
raises average branch coverage from 68.6\,\% to 91.6\,\% at an LLM
overhead of 14.3\,\% of runtime. Future work includes reducing scheduler
latency through speculative decoding, extending the execution model to
heterogeneous accelerators, and applying DNSSE to safety-critical
autonomous systems.

\paragraph*{Acknowledgments.}
The authors thank the engineering team at Quandary Peak Research for
their support and infrastructure provisions during the empirical
evaluation.

\end{document}